\documentclass[runningheads]{llncs}

\usepackage[T1]{fontenc}
\def\doi#1{\href{https://doi.org/\detokenize{#1}}{\url{https://doi.org/\detokenize{#1}}}}
\usepackage{graphicx}
\usepackage{listings}
\usepackage{anyfontsize}
\usepackage{multirow}
\usepackage{graphicx} 
\usepackage{comment}
\usepackage{color}
\usepackage{algorithm}
\usepackage[noend]{algpseudocode} 
\usepackage{amssymb,amsmath}
\usepackage{hyperref}

\usepackage{ifthen}
\newboolean{shortver}
\setboolean{shortver}{true}

\begin{document}
\title{A Lower Bounding Framework for Motion Planning amid Dynamic Obstacles in 2D}
%
%

\author{Zhongqiang Ren\inst{1}
\and
Sivakumar Rathinam\inst{2}
\and
Howie Choset\inst{1}
}

\institute{Carnegie Mellon University, Pittsburgh, PA 15213, USA
\email{\{zhongqir,choset\}@andrew.cmu.edu}\\
\and
Texas A\&M University, College Station, TX 77843-3123.
\email{srathinam@tamu.edu}}
\maketitle              

%

\newcommand{\red}{\color{red}}
\newcommand{\green}{\color{green}}
\newcommand{\blue}{}

\thispagestyle{plain}
\pagestyle{plain}
\pagenumbering{arabic}

\begin{abstract}
	This work considers a Motion Planning Problem with Dynamic Obstacles (MPDO) in 2D that requires finding a minimum-arrival-time collision-free trajectory for a point robot between its start and goal locations amid dynamic obstacles moving along known trajectories. Existing methods, such as continuous Dijkstra paradigm, can find an optimal solution by restricting the shape of the obstacles or the motion of the robot, while this work makes no such assumptions. Other methods, such as search-based planners and sampling-based approaches can compute a feasible solution to this problem but do not provide approximation bounds. Since finding the optimum is challenging for MPDO, this paper develops a framework that can provide tight lower bounds to the optimum. These bounds act as proxies for the optimum which can then be used to bound the deviation of a feasible solution from the optimum. To accomplish this, we develop a framework that consists of (i) a bi-level discretization approach that converts the MPDO to a relaxed path planning problem, and (ii) an algorithm that can solve the relaxed problem to obtain lower bounds. We also present numerical results to corroborate the performance of the proposed framework. These results show that the bounds obtained by our approach for some instances are up to twice tighter than a baseline approach showcasing potential advantages of the proposed approach. 


\end{abstract}

\section{Introduction}

\graphicspath{{figures/}}

The motion planning problem of finding a collision-free trajectory for a robot in the presence of obstacles is one of the most fundamental problems in Robotics \cite{choset2005principles,halperin2017algorithmic,latombe2012robot,lavalle2006planning}. In this article, we consider a Motion Planning Problem with Dynamic Obstacles (MPDO) in 2D where the goal is to find a collision-free trajectory for a point robot between its start location and destination in the presence of dynamic obstacles such that the arrival time of the robot is minimized. The obstacles can be of arbitrary shapes moving along known trajectories in the workspace.
If the obstacles are static, depending on the shape of the obstacles, there are several algorithms to either find an optimal or bounded sub-optimal solutions in polynomial time \cite{choset2005principles,latombe2012robot,lavalle2006planning}.
In the presence of dynamic obstacles, this motion planning problem is known to be NP-Hard~\cite{canny1987new}.

There is currently no algorithm available in the literature to find an optimal solution for the MPDO with generic obstacles following known but arbitrary trajectories.
There are methods, however, for finding an optimal solution to some special cases of MPDO. By assuming that obstacles are convex polygons moving along fixed directions with constant speeds, minimum-time trajectories can be found by constructing an 
``accessibility graph''~\cite{fujimura1993planning} in 3D with time added as the third dimension.
If the shapes of the obstacles are polygons, continuous Dijkstra paradigms~\cite{fujimura1994motion,hershberger1999optimal} can be leveraged to compute an optimal solution. When both the obstacles and the path of the robot are rectilinear, an improved version of the continuous Dijkstra paradigm with reduced time complexity has been developed recently in \cite{maheshwari2020shortest}. While these methods find optimal solutions in theory, we are not aware of any implementations of these algorithms or any numerical results on the performance of these algorithms.

While finding the optimum is challenging, there are many algorithms in the literature that can compute a \emph{feasible} solution for MPDO. By discretizing the workspace into a graph, search-based planners~\cite{phillips2011sipp} can find an optimal solution within the graph, whose quality is determined by the resolution of the discretization.
Path-velocity decomposition methods~\cite{kant1986toward} can also be leveraged to find a feasible solution by first finding a path among static obstacles and then finding the speeds along the path to avoid the dynamic obstacles. Sampling-based methods \cite{frazzoli2002real,hsu2002randomized} have been applied for generalizations of MPDO with motion constraints and can be used to find feasible solutions. While one can find feasible solutions, there are currently no {\it a-posteriori} or {\it a-priori} bounds that quantify the deviation of these feasible solutions from the optimum. 


In the absence of methods for finding the optimum, lower bounds can be used as {\it proxies} to the optimum. Baseline approaches that find trivial lower bounds are always possible, i.e., by either removing the dynamic obstacles or relaxing the collision avoidance constraints. However, these bounds tend to be far away from the optimum, and therefore, do not serve as good estimates of the optimum. The goal of this paper is to develop a framework that can provide tight estimates of the optimum to the MPDO in 2D (Fig. \ref{lbmp:fig:discretization}). While doing so, we make no assumptions on the shape of the dynamic obstacles or their motion. The framework we develop consists of two parts.
In the first part, we present a bi-level discretization approach that relaxes some of the constraints of the MPDO and converts it to a lower bounding path planning problem on a discrete graph, which is the main contribution of this work. The second part then solves the path planning problem using an A*-based~\cite{astar} algorithm tuned to handle the constraints in the lower bounding problem.

The proposed bi-level discretization approach (Fig. \ref{lbmp:fig:discretization}) simultaneously handles the collision avoidance constraints and the generation of tight lower bounds. The higher level of the discretization approach partitions the 2D workspace into cells. This level allows us to relax and formulate the collision avoidance constraints between the robot and the dynamic obstacles. The lower level divides each boundary of a cell into smaller sub-segments allowing a robot to either wait on a sub-segment or travel only between two sub-segments belonging to a cell. Each sub-segment is then treated as a node and the lower bounding problem is formulated as a search on a graph consisting of all the nodes representing the sub-segments. While the size of the cells at the higher level controls how the collision avoidance constraints are strictly enforced, the size of the sub-segments at the lower level determines the closeness of the computed trajectories to an optimal solution. We theoretically show that our framework generates lower bounds and then present numerical results to corroborate the performance of the proposed approach.

\begin{figure}[H]
	\centering
	\includegraphics[width=0.7\linewidth]{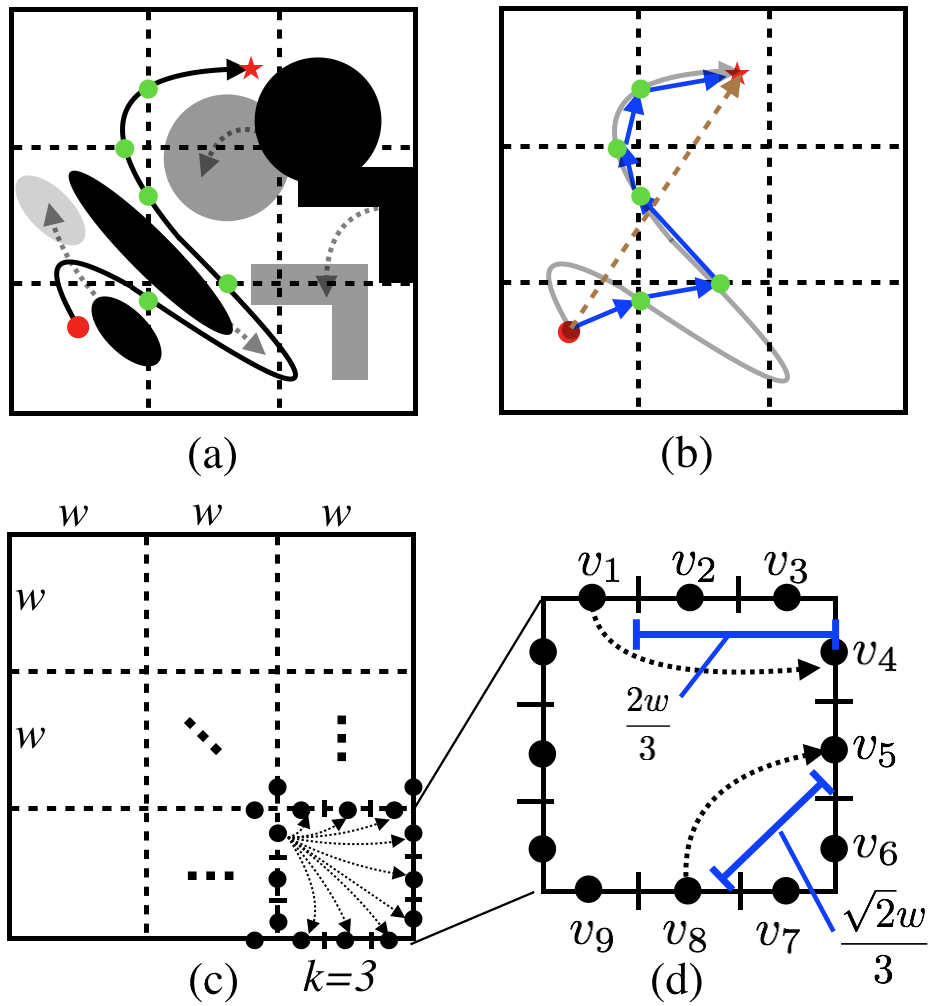}
	\caption{(a) An illustration of MPDO in 2D. The robot moves from the red dot (the starting location) to the red star (the goal location) along a collision-free path. (b) The brown dashed line visualizes the lower bounding solution computed by a baseline method ignoring the dynamic obstacles (Sec.~\ref{lbmp:sec:result}), which is not tight in general. The blue lines illustrate the lower bounding solution of our approach.
	(c) Bi-level discretization: the workspace is divided at the higher level into cells of size $w$. At the lower level, the boundaries of each cell is divided into $k$ sub-segments of equal size. Each sub-segment forms a node in a graph and the lower bounding problem is formulated on this graph. (d) Examples of lower bounding distances between vertices (i.e., sub-segments).
	By tuning the hyper-parameters $w,k$, the tightness of the lower bound can be adjusted at the expense of computational cost.}
	\label{lbmp:fig:discretization}
\end{figure}

\section{Preliminaries}\label{lbmp:sec:preliminaries}

This work considers a workspce $\mathcal{W}=[0,L]\times[0,L]$ and continuous time $t \in [0,T]$, with $L$ and $T$ being finite positive real numbers.
Each obstacle is moving along some known trajectory within the time range $[0,T]$, with $N_{obs}$ denoting the total number of obstacles.
Obstacles can have different and arbitrary shapes, and can overlap with each other at any time.
{\blue Note that the aforementioned notion of obstacles include static obstacles, which follow trajectories that stay in place within $[0,T]$.}
Let $O^i(t)\subset \mathcal{W}, i=1,2,\dots,N_{obs}$ denote the subset of the workspace occupied by the $i$-th obstacle at time $t$. Let $s,d \in \mathcal{W}$ denote the start and destination of a point robot respectively.
Let $p:[0,T]\rightarrow \mathcal{W}$ represent a trajectory of the robot from $s$ to $d$, and $p$ is collision-free if the robot does not enter the interior of any obstacle at any time.
The robot can either wait in place or move in any direction with speed no larger than $V_{max}$.
Let $p^*$ denote a collision-free trajectory with the minimum arrival time $C^*$ at $d$, and $C^*$ is also referred to as the optimal cost. The Lower Bounding Problem (LBP) aims to compute a lower bound (i.e., underestimate) of $C^*$. In Sec. \ref{lbmp:sec:generate_lbp}, we discuss the bi-level discretization approach that leads to the formulation of the LBP. In Sec. \ref{lbmp:sec:algo} we then present an algorithm called LB-A$^*$ to solve the LBP.



\section{A Lower Bounding Problem Formulation}\label{lbmp:sec:generate_lbp}

\graphicspath{{figures/}}


\subsection{Bi-Level  Discretization}\label{lbmp:sec:generate_lbp:discretize}

\subsubsection{Graph Vertices:}
As shown in Fig.~\ref{lbmp:fig:discretization} (c), at the higher level, $\mathcal{W}$ is decomposed into $n \times n$ (squared) \emph{cells} of size $w\times w$ ($n=L/w$).
Each cell is enclosed by four \emph{boundaries} or \emph{line segments} that are perpendicular to each other.
At the lower level, a line segment is evenly divided into $k$ \emph{sub-segments} of length $\frac{w}{k}$.
Both line segments and sub-segments are closed sets of points (i.e., including the ending point).
Two sub-segments are said to be \emph{next} to each other if they have one common ending point.
Let $V_o$ denote a set of (graph) vertices, where each vertex corresponds to a sub-segment in a cell.
Let $V:=V_o \cup \{s,d\}$, where both $s$ and $d$ can be regarded as a special sub-segment containing only a single point.

Let $LS(v)$ (and $SS(v)$), $\forall v\in V$ denote the set of points within the line segment (and the sub-segment respectively) corresponding to $v$. Note that this definition includes the cases where $v=s$ and $v=d$: $LS(s)=SS(s)=\{s\}$ and $LS(d)=SS(d)=\{d\}$.
Also, let $\mathcal{LS} := \bigcup_{v\in V}LS(v)$ represent the set of all points that lie on any line segment. Similarly, let $\mathcal{SS} := \bigcup_{v\in V}SS(v)$, and note that $\mathcal{SS}=\mathcal{LS}$.

\subsubsection{Graph Edges:}
{\blue For any pair of distinct vertices $v_i,v_j \in V$ within the same cell, $v_i$ and $v_j$ are connected with an (un-directed) edge if they do not belong to the same line segment (i.e., $LS(v_i)\neq LS(v_j)$).
For example, in Fig.~\ref{lbmp:fig:discretization} (d), vertices $v_2,v_3$ are not connected since they belong to the same line segment, and $v_3, v_4$ are connected as they belong to different line segments.
More discussion on the edge connectivity is presented in Sec.~\ref{lbmp:sec:discussion}.}
The edge joining any pair of vertices $(v_i,v_j)$ represents a move action of the robot between $(v_i,v_j)$.
Since each vertex $v\in V$ corresponds to a unique sub-segment $SS(v)$, to simplify the presentation, we also say an edge $(v_1,v_2)$ connects two sub-segments $SS(v_1)$ and $SS(v_2)$. Let $E_{move}$ denote the set of all edges connecting a pair of sub-segments in all the cells.


\subsubsection{Edge Costs:}
The cost of an edge $(v_1,v_2) \in E_{move}$ is defined to be an \emph{underestimate} of any possible transition time (i.e., cost) between any pair of points in the respective sub-segments $SS(v_1),SS(v_2)$ connected by the edge.
Formally, for an edge $e=(v_1,v_2) \in E_{move}$,
\begin{gather}\label{lbmp:eqn:edge_cost}
    cost(e) := \frac{\min_{x_1,x_2} || x_1 - x_2 ||}{ V_{max} }, x_1 \in SS(v_1), x_2 \in SS(v_2).
\end{gather}
A few examples are shown in Fig.~\ref{lbmp:fig:discretization} (d): $cost(v_3,v_4)=0$, $cost(v_1,v_4)=2w/3$ and $cost(v_8,v_5)=\sqrt{2}w/3$.



\subsection{Reachable Time Intervals and Self-Loops}\label{lbmp:sec:generate_lbp:reachable}

To consider the collision avoidance requirement between the robot and the obstacles (both static and dynamic obstacles), for each vertex $v \in V$ and a time point $t\in [0,T]$, vertex $v$ and $LS(v)$ are said to be \emph{non-reachable} if {\blue the entire line segment lies inside the union of obstacles at time $t$ (i.e., $LS(v) \subseteq \bigcup_{i \in \{1,2,\dots,N_{obs}\}}O^i(t)$).}
Otherwise, vertex $v$ and $LS(v)$ are said to be \emph{reachable} at time $t$.
Additionally, for each vertex $v \in V$, the sub-segment $SS(v)$ is said to be reachable (or non-reachable) if $LS(v)$ is reachable (or non-reachable).
{\blue Here, the possible collision between the robot and any obstacle is only considered at vertices and is ignored along all edges (i.e., ignored during the transition between vertices).
This notion of collision defined in the graph relaxes the collision avoidance requirement in the original continuous problem, and is hereafter referred to as \textbf{Relaxation-1}.}
Note that this relaxation is allowed since our objective is to find a lower bounding solution of $C^*$, instead of a feasible solution.

\begin{figure}[t]
	\centering
	\includegraphics[width=0.75\linewidth]{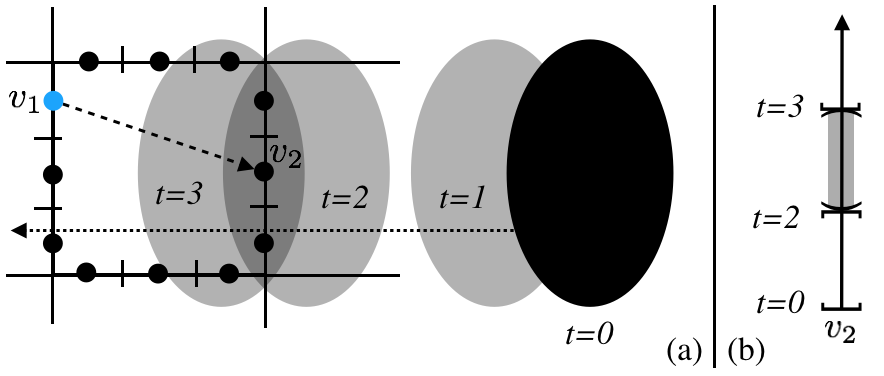}
	\caption{(a) An obstacle (the black oval) that moves from the right to the left along some known trajectory. (b) The reachable intervals of vertex $v_2$ are $[0,2],[3,T]$ while the non-reachable interval is $(2,3)$.}
	\label{lbmp:fig:intervals}
\end{figure}

Let a \emph{reachable time interval} $Itv(v), v \in V$ be a maximal contiguous time range $[t^a,t^b]$, where $v$ is reachable at any time point $t\in [t^a,t^b]$.
Given the trajectories of all the obstacles, the reachable time intervals of all vertices can be computed.
In general, there are multiple (for example $J$, a finite positive integer) reachable time intervals at a vertex $v$, and we use subscript $j$ in notation $Itv_j(v), j=1,2,\dots,J$ to denote each of these reachable intervals at vertex $v$.
The robot can only reach a vertex $v \in V$ if the arrival time $t$ is within some $Itv_j(v)$.
The complements of reachable intervals are called \emph{non-reachable intervals}.
For example, in Fig.~\ref{lbmp:fig:intervals}, the reachable intervals of vertex $v_2$ are $[0,2],[3,T]$, and the only non-reachable interval of $v_2$ is $(2,3)$.

After the robot has reached a vertex, the robot is allowed to stay at this vertex for an arbitrary amount of time, ignoring any non-reachable interval at this vertex in the future.
This is also a relaxation of the collision avoidance constraint between the robot and the obstacles, which is hereafter referred to as \textbf{Relaxation-2}.
For example, in Fig.~\ref{lbmp:fig:intervals}, if the robot reaches $v_2$ at a time point within $[0,2]$, then the robot is allowed to stay at $v_2$ during the non-reachable interval $(2,3)$.

The wait action of the robot can be described as a special type of edge, the \emph{self-loop} in the graph, whose cost can be any positive number (indicating the wait time of the robot at the vertex) and needs to be determined during the planning process.
As presented in Sec.~\ref{lbmp:sec:algo}, the proposed planner that solves the LBP will determine the self-loop cost during the planning process.
Finally, let $E_{wait}$ denote the set of all self-loops corresponding to all vertices in $V$.

\vspace{2mm}
{\blue
\noindent\textbf{Remark.} The reachable time intervals are similar to the notion of ``safe intervals'' in SIPP~\cite{phillips2011sipp}.
The main difference is the newly introduced Relaxation-1,2, which allow us to obtain lower bounds by relaxing the collision avoidance constraint at a vertex or during the transition between vertices and intervals.
Since we are looking for lower bounding solutions (which can collide with the obstacles in the continuous space and time and thus ``unsafe''), we use the term ``reachable time intervals'' to highlight the difference.
}


\subsection{Lower Bounding Problem Definition}\label{lbmp:sec:generate_lbp:def}
Let edge set $E:=E_{move}\cup E_{wait}$, and define graph $G:=(V,E)$.
Let $Itv(G)$ denote the set of all reachable intervals of all vertices in $G$.
When parameters $(w,k)$ are given, the corresponding graph $G$ as well as $Itv(G)$ are well defined, which specifies a $(w,k)$-Lower Bounding Problem (referred to as $(w,k)$-LBP):
\vspace{1mm}
\begin{definition}
	A $(w,k)$-LBP requires finding a minimum-cost trajectory $p$ from $s$ to $d$ in $G$ such that the arrival time at each vertex $v$ along $p$ is within some reachable time interval $Itv(v)$.
\end{definition}
We will prove in Sec.~\ref{lbmp:sec:analysis} that the cost of an optimal solution trajectory to a $(w,k)$-LBP is guaranteed to be a lower bound of $C^*$, and hence the name ``lower bounding problem''.

\section{Lower Bounding A*}\label{lbmp:sec:algo}

\subsection{LB-A* Overview}
Given a $(w,k)$-LBP (which includes the corresponding graph $G$ and $Itv(G)$), we develop an A*-like~\cite{astar} graph search algorithm called LB-A* (Lower Bounding A*) to compute an optimal solution, which is shown in Alg.~\ref{lbmp:alg:lbastar}.

\begin{algorithm}[tb]
	\caption{Pseudocode for LB-A*}\label{lbmp:alg:lbastar}
	\begin{algorithmic}[1]
		\State{$g(v)\gets\infty, \forall v\in V$}
		\State{$g(s)\gets0$, and add $s$ to OPEN}
		\While{OPEN not empty} \Comment{Main search loop}
		\State{$v \gets$ OPEN.pop() }
		\State{\textbf{if} $v=d$ \textbf{then}}
		\State{\indent \textbf{return} \text{\textit{Reconstruct}($v$)}}
		\ForAll{$u \in$ \textit{Neighbor}($v$)}
		\State{$g' \gets$ \textit{EarliestReach}($v,u$)}
		\State{\textbf{if} $g(u) \leq g'$}
		\State{\indent \textbf{continue}}\Comment{End of this iteration}
		\State{$g(u) \gets g'$, $parent(u) \gets v$}
		\State{Add $u$ to OPEN}
		\EndFor
		\EndWhile \label{}
		\State{\textbf{return} Failure}
	\end{algorithmic}
\end{algorithm}

Similar to A*, let $g(v),\forall v\in V$ denote the cost-to-come (i.e., the earliest arrival time at $v$), which is initialized to $\infty$ for all vertices with the exception that $g(s)$ is set to zero.
Let $h(v)$ denote the heuristic value, which underestimates the cost-to-go from $v$ to the destination $d$.
Also, let OPEN denote a priority queue of candidate vertices that will be selected and expanded by the algorithm at any time during the search.
OPEN prioritizes candidate vertices $v$ based on their $f$-values, which are defined as $f(v):=g(v)+h(v)$.
Initially, $s$ is inserted into OPEN (line 2) with $f(s)=h(s)$ (since $g(s)=0$).

In each search iteration (lines 3-12), a vertex in OPEN with the minimum $f$-value is selected for expansion.
If $v$ is the same as $d$, a solution trajectory is found, which is reconstructed by iteratively backtracking the parent pointers of vertices.
The cost of this trajectory is guaranteed to reach the minimum.
Otherwise, $v$ is expanded by examining each of its neighboring vertices in $G$ (denoted as \textit{Neighbor($v$)}).

During the expansion, for each $u \in $ \textit{Neighbor($v$)} (lines 7-12), the earliest possible reachable time (denoted as $g'$) from $v$ to $u$ is computed via procedure \textit{EarliestReach}, which is explained later.
Then, $g'$ is used to update $g(u)$ if $g(u) > g'$, which means $g'$ leads to a cheaper trajectory to reach $u$ via $v$, and $u$ is inserted into OPEN for future expansion.

\subsection{Compute Earliest Reachable Time}
This section revisits reachable intervals and then presents the procedure \textit{EarliestReach}.
For each $v\in V$, the set of reachable intervals $\{Itv(v)\}$ can be pre-computed based on the trajectories of the obstacles, and by definition, no two intervals in $\{Itv(v)\}$ can overlap with each other.
For each $v$, sort $\{Itv(v)\}$ based on their starting time points from the minimum to the maximum, and denote the $j$-th interval in $\{Itv(v)\}$ as $Itv_j(v), j=1,2,\dots J$, where $J$ is a finite number.

To compute the earliest reachable time from vertex $v$ to $u$, \textit{EarliestReach}($v,u$) iteratively checks each reachable interval $Itv_j(u)=[t^a_j,t^b_j],j=1,2,\dots,J$ to find the first interval $Itv_{j'}(u)=[t^a_{j'},t^b_{j'}]$ such that the ending time point $t^b_{j'} \geq g(v)+cost(v,u)$.
The robot is guaranteed to be able to reach $u$ within this reachable time interval via waiting at $v$ and then moving to $u$, and the earliest reachable time is $\max\{t^a_{j'}, g(v)+cost(v,u)\}$.
In other words, when $t^a_{j'} > g(v)+cost(v,u)$, the robot waits at $v$ for an amount of time $t^a_{j'} - (g(v)+cost(v,u))$ and then moves from $v$ to $u$ using an amount of time $cost(v,u)$.
Note that the robot can wait at $v$ for an arbitrary amount of time (according to Relaxation-2), and the potential collision along the edge is ignored (according to Relaxation-1).




\section{Analysis}\label{lbmp:sec:analysis}

\graphicspath{{figures/}}

\subsection{Lower Bounds}
This section introduces some definitions and then shows that the solution cost to the aforementioned $(w,k)$-LBP problem is a lower bound of $C^*$.

\begin{definition}[Line Segment Indicator Function]
	Let $I_{LS}(x),x \in \mathcal{W}$ denote the line segment indicator function: 
	\begin{gather}
        I_{LS}(x)=
        \begin{cases}
        LS(x) &\mbox{if } x\in \mathcal{LS} \\
        \emptyset & \mbox{if } x\notin \mathcal{LS}.
        \end{cases}\nonumber
	\end{gather}
\end{definition}
Note that in the above definition, we abuse the notation a bit to simplify our presentation: here, notation $LS(x)$ denotes the line segment that contains point $x\in \mathcal{W}$ when $x$ is within some line segment, while in the previous section, notation $LS(v), v\in G$ is only defined over the vertices in $G$.

As a special case, if $x$ is an ending point of a line segment that is shared by multiple line segments, then $LS(x)$ can denote an arbitrary one of those line segments.
An illustration of $I_{LS}$ is shown in Fig.~\ref{lbmp:fig:proof} (a).

\vspace{1.5mm}
\begin{definition}[Arrival Times]\label{lbmp:def:arrival_times}
	When an optimal solution trajectory $p^*$ to the continuous problem exists, arrival times $\tau_j,j=0,1,\dots,j_{max}$ along $p^*$ are real numbers that are defined as follows. If $j=0$ then $\tau_0=0$; {\blue Otherwise $\tau_{j+1} = \inf\{t \,|$ $t>\tau_j$, $I_{LS}(p^*(t)) \neq \emptyset,$ $I_{LS}(p^*(t)) \neq I_{LS}(p^*(\tau_j)) \}$.}
\end{definition}
\vspace{1.5mm}
{\blue
In the above definition, condition $t>\tau_j$ guarantees $\tau_{j}$ increases when $j$ increases. Condition $I_{LS}(p^*(t)) \neq \emptyset$ ensures $p^*(\tau_{j})$ for any $j$ is within some line segment. 
Condition $I_{LS}(p^*(t)) \neq I_{LS}(p^*(\tau_j))$ ensures the corresponding line segments of $\tau_j$ and $\tau_{j+1}$ are not the same.}

\begin{figure}[t]
	\centering
	\includegraphics[width=0.85\linewidth]{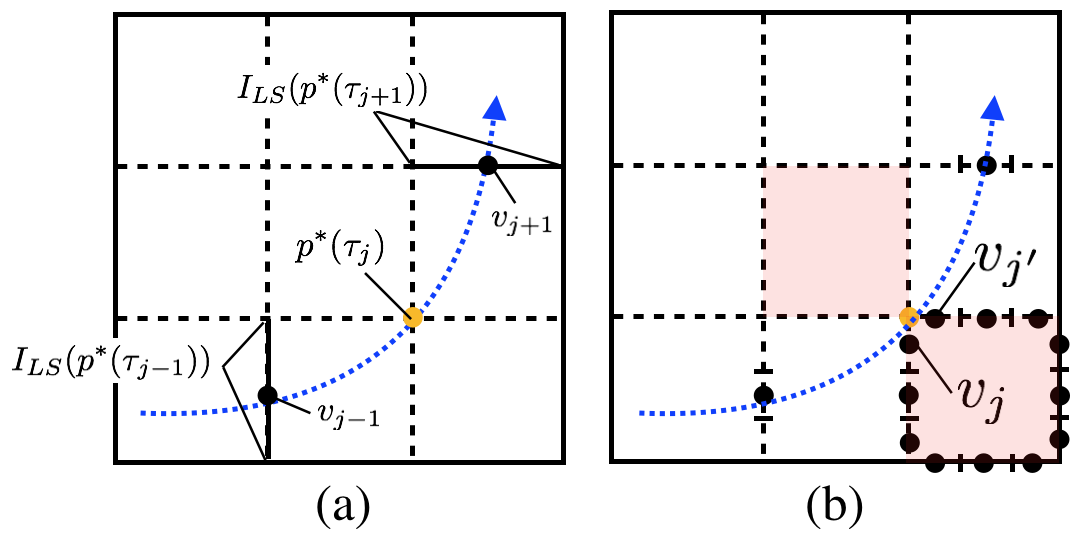}
	\caption{(a) visualizes the notion of the line segment indicator function. (b) visualizes an edge case, where the duplication trick is used to ensure that the corresponding sub-segments contained in two subsequent $I_{LS}(p^*(\tau_j))$ are connected by an edge in $G$. See text for more details.}
	\label{lbmp:fig:proof}
\end{figure}

There is a special case that requires additional discussion before introducing Lemma~\ref{lbmp:lem:p_in_G_exists}, which is conceptually visualized as the yellow point in Fig.~\ref{lbmp:fig:proof} (a).
When $p^*$ goes through a (corner) point $x$ that is shared by four adjacent cells, a ``duplication'' trick is required to make sure that the corresponding sub-segments contained in two subsequent $I_{LS}(p^*(\tau_j))$ are connected by an edge in $G$.
Specifically, let $I_{LS}(p^*(\tau_j))$ denote a line segment that is within the same cell as $I_{LS}(p^*(\tau_{j-1}))$.
Then, if $I_{LS}(p^*(\tau_j))$ is not within the same cell as $I_{LS}(p^*(\tau_{j+1}))$, duplicate an additional time point $\tau_j'=\tau_j$ and let $I_{LS}(p^*(\tau_j'))$ denote the line segment that is within the same cell as $I_{LS}(p^*(\tau_{j+1}))$.
As an illustration, in Fig.~\ref{lbmp:fig:proof} (b), the line segment corresponding to $v_j$ is such a choice (not the only choice) for $I_{LS}(p^*(\tau_j))$, and the line segment corresponding to $v_{j'}$ is the duplication. (Note that the cost of edge $(v_j,v_{j'})$ is zero.)

\begin{lemma}\label{lbmp:lem:p_in_G_exists}
	When $p^*$ exists, a corresponding trajectory $p$ in $G$ can be constructed such that every two subsequent vertices in $p$ are connected by an edge in $G$.
\end{lemma}


\begin{lemma}\label{lbmp:lem:k_max_exists}
	When $p^*$ exists, there is a finite $j_{max}$ such that $p^*(\tau_{j_{max}})=d$.
\end{lemma}

\begin{theorem}\label{lbmp:thm:exist_corresponding_sol}
    Given a continuous problem with an optimal solution $p^*$ (with the minimum cost $C^*$), and a $(w,k)$-LBP with graph $G$, there exists a corresponding feasible solution $p$ in $G$ such that $cost(p) \leq C^*$.
\end{theorem}
\begin{proof}
    By Lemma~\ref{lbmp:lem:k_max_exists}, $j_{max}$ is finite.
    If $j_{max}= 1$, then $s,d$ are within the same cell, and $\tau_j,j=0,1$ corresponds to $LS(s),LS(d)$ respectively.
    In this trivial case, a straight-line trajectory that directly connects $s,d$ exists and its cost is no larger than $C^*$.
    
    If $j_{max} > 1$, then $p^*$ must intersect with at least one line segment.
    For any $j\in\{0,1,\dots,j_{max}-1\}$, at time $\tau_j$, point $p^*(\tau_j)$ is collision-free, since the point is part of $p^*$ (which is a collision-free optimal trajectory to the continuous problem).
	Line segment $I_{LS}(p^*(\tau_j))$ is thus reachable.
	For the same reason, line segment $I_{LS}(p^*(\tau_{j+1}))$ is also reachable.
	By Lemma~\ref{lbmp:lem:p_in_G_exists}, a trajectory $p$ in $G$ corresponding to $p^*$ exists.
	Let $v_j$ and $v_{j+1}$ denote the vertices in the graph $G$ along $p$ such that $p^*(\tau_{j}) \in LS(v_{j})$ and $p^*(\tau_{j+1}) \in LS(v_{j+1})$.
	By definition in Sec.~\ref{lbmp:sec:generate_lbp:reachable}, the line segments $LS(v_{j})$ and $LS(v_{j+1})$ are reachable at time $\tau_j$ and $\tau_{j+1}$ respectively.
	In addition, $cost(v_j,v_{j+1}) \leq \tau_{j+1}-\tau_{j}$ (by Equation~\ref{lbmp:eqn:edge_cost}).
	The arrival times at each vertex along $p$ can be constructed by letting the robot reach each $v_{j+1}$ at time $\tau_{j+1}$ in $G$ via wait and move actions, for all $j=0,1,\dots,j_{max}-1$. (The wait time at each $v_j$ is $\tau_{j+1}-\tau_{j}-cost(v_j,v_{j+1})$.)
\end{proof}

\begin{lemma}\label{lbmp:lemma:lb_astar_optimal}
    Given a $(w,k)$-LBP, LB-A* computes an optimal trajectory in $G$ when heuristic values are admissible (i.e., $h(v), \forall v\in G$ underestimates the cost of an optimal trajectory from $v$ to $d$ in $G$).
\end{lemma}

\begin{theorem}\label{lbmp:thm:lb}
Given a continuous problem with the minimum cost $C^*$, the cost of the solution trajectory $p$ (denoted as $cost(p)$) computed by LB-A* to any $(w,k)$-LBP satisfies $cost(p) \leq C^*$ (i.e., $cost(p)$ is a lower bound of $C^*$).
\end{theorem}

\begin{proof}
By Theorem~\ref{lbmp:thm:exist_corresponding_sol}, there exists a corresponding trajectory $p'$ in $G$ to an optimal trajectory $p^*$ to the continuous problem, and $cost(p')\leq cost(p^*)=C^*$.
By Lemma~\ref{lbmp:lemma:lb_astar_optimal}, LB-A* computes an optimal solution $p$ in $G$ and thus $cost(p) \leq cost(p')$. Therefore, $cost(p) \leq C^*$.
\end{proof}

\subsection{Computational Complexity}\label{lbmp:sec:analysis:compute}

\subsubsection{Graph Size:}
In the $(w,k)$-LBP problem formulation, to discretize the workspace, there are totally $n^2$ cells.
Within each cell, there are $4k$ sub-segments.
Since each sub-segment that is not on the borders of the workspace is shared by two cells, there are totally $2k(n^2+n)$ sub-segments in the workspace.
Thus, the number of vertices in $G$ is $|V|=2k(n^2+n)=O(kn^2)$.

Within a cell, each sub-segment is connected with at most $3k$ other sub-segments.
There are at most $4k\times3k$ edges within a cell.
Since the connectivity between sub-segments within each cell is defined in the same way, only one copy of the edges needs to be stored.
Therefore, graph $G$ requires a storage of size $|V|+|E_{move}|=O(2k(n^2+n) + 12k^2)=O(kn^2 + k^2)$.

{\blue As this work does not assume the shape or the trajectory of the obstacles, it is hard to bound the size of $Itv(G)$ and analyze the computational complexity of $Itv(G)$. We discuss potential future work in Sec.~\ref{lbmp:sec:conclude}.}

\subsubsection{Search Branching Factor:}
As aforementioned, each sub-segment is connected with at most $3k$ other sub-segments within a cell, and each sub-segment (not on the borders of $\mathcal{W}$) is shared by two adjacent cells.
Thus, for each sub-segment, there are at most $6k$ neighbors, which is the branching factor that affects the search efficiency of LB-A*.
As shown in the ensuing section, there is a trade-off between computing tighter lower bounds (increasing $n$ and $k$) and expanding fewer vertices (decreasing $n$ and $k$).

\section{Discussion}\label{lbmp:sec:discussion}
{\blue

\subsection{Edges Between Sub-Segments}\label{lbmp:sec:discuss:edges}
First, we explain the reason for not connecting vertices within the same line segment as mentioned in Sec.~\ref{lbmp:sec:generate_lbp:discretize}.
With Equation~\ref{lbmp:eqn:edge_cost}. the edges connecting any two adjacent vertices have zero cost. If all the adjacent vertices are connected, then the optimum for the lower bounding problem can have zero cost, which is a trivial lower bound.
Additionally, for non-adjacent vertices that lie within the same line segment, there is no need to connect them for the following reason.
With the LBP formulation and LB-A*, we only need to consider the case where there is only one transition inside a cell before leaving it.
In other words, there is no need to consider the case where the trajectory intersects with the same line segment multiple times, since (i) the robot can wait for any amount of time after its arrival at a vertex (due to Relaxation-2), and (ii) the resulting trajectory is still a lower bound based on the proof in Sec.~\ref{lbmp:sec:analysis}.
Based on this observation, the edges between non-adjacent vertices that lie within the same line segment can be omitted, which can help reduce the branching factor during the search.

\subsection{Adding Expansion Constraints}\label{lbmp:sec:discuss:exp_cstr}
As mentioned in the previous sub-section, there is no need to consider the case where the trajectory in $G$ intersects with the same line segment multiple times.
We can leverage this observation to enforce that each cell is traversed for only once during the search, which has the potential to provide a tighter lower bound.
Specifically, the \textit{Neighbor(v)} procedure at line 7 in Alg.~\ref{lbmp:alg:lbastar} is modified as follows.
First, let notation $cell(u,v)$ denote the cell that is traversed by edge $(u,v) \in E_{move}$ (i.e., $LS(u)$ and $LS(v)$ are boundaries of $cell(u,v)$, which determines a unique cell).
When generating the neighbor vertices of $v$ in procedure \textit{Neighbor(v)}, we enforce the constraint that: for each generated neighbor vertex $v'$, $cell(v,v') \neq cell(parent(v),v)$.
In other words, edge $(v,v')$ cannot traverse the same cell as edge $(parent(v),v)$ does.
As a special case, in the first iteration of the search (i.e., when $v=s$ and $parent(s)$ does not exist), we do not enforce this constraint when expanding $s$.

}

\section{Numerical Results}\label{lbmp:sec:result}

\graphicspath{{figures/}}

For all the tests in this work, the workspace is of size $L\times L$ with $L=1$ and $V_{max}=0.03$.
In LB-A*, the expansion step uses the one as described in Sec.~\ref{lbmp:sec:discuss:exp_cstr}, and heuristic values of all vertices are simply set to zero, which are admissible.
{\blue
The \textbf{baseline} approach used for comparison in all the experiments is: (i) constructing a visibility graph~\cite{lozano1979algorithm} among the static obstacles while ignoring all the dynamic obstacles, and (ii) finding a shortest path connecting $s$ and $d$ within the visibility graph.
Please refer to our video (\url{https://youtu.be/wf76WJj7KtQ}) for visualization of the experimental results.
}

\subsection{Experiment 1: Simple Instance with Known $C^*$}

\begin{figure*}[tb]
	\centering
	\includegraphics[width=\linewidth]{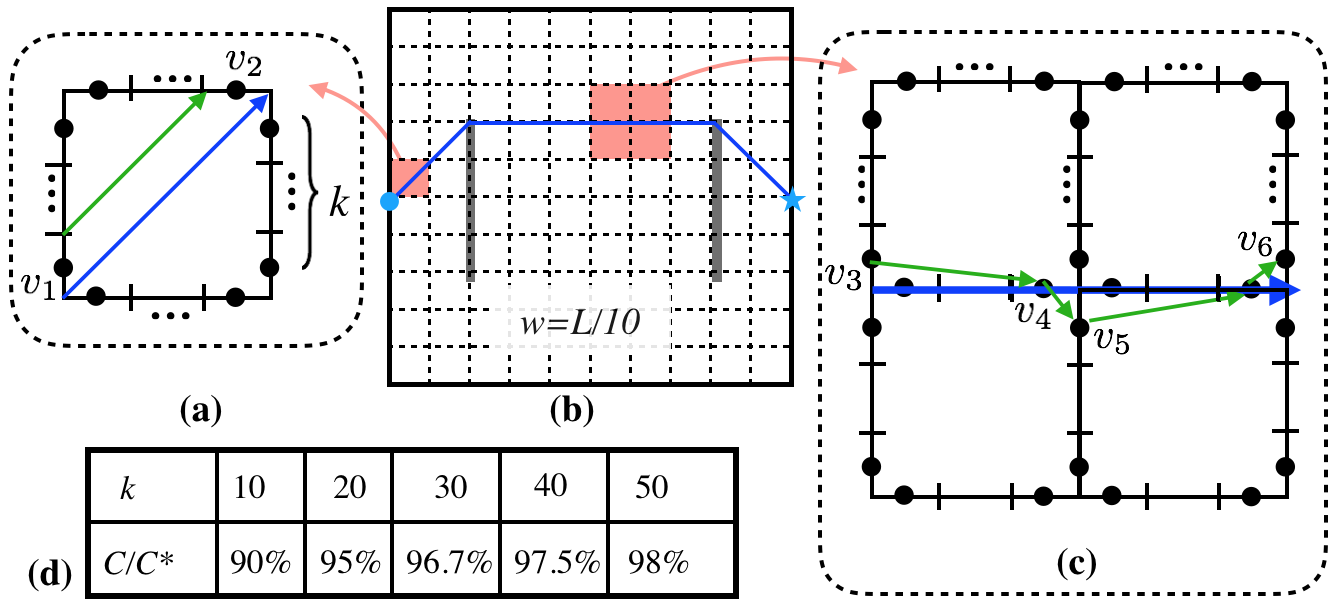}
	\caption{Experiment 1.
	Fig. (b) shows the workspace with two bar-like static obstacles, and an optimal trajectory $p^*$ from the start (the blue dot on the left) to the destination (the blue star on the right).
	Fig. (a) and (c) visualize $p^*$ in blue and the solution (trajectory) computed by LB-A* in green.
	Fig. (d) shows the ratios $C/C^*$ with varying $k$, where $C$ is the cost of the solution computed by LB-A*, and $C^*$ is the cost of $p^*$.
	This figure shows that, by increasing $k$ (the number of sub-segments), the lower bound computed by LB-A* becomes tighter.}
	\label{lbmp:fig:converge}
\end{figure*}

To begin with, we construct a simple test instance as shown in Fig.~\ref{lbmp:fig:converge} (b).
There are two bar-like \emph{static} obstacles with negligible width, and no dynamic obstacles.
Since there is no dynamic obstacle, the solution computed by the baseline $p^*$ is an optimal trajectory.
In Fig.~\ref{lbmp:fig:converge} (a,b,c), $p^*$ is visualized as the blue lines.
The true optimal cost $C^*$ can be calculated, which is $\frac{0.4\sqrt{2}+0.6}{V_{max}}=38.853$.

To verify our approach, a LBP is formulated by discretizing the workspace into $10\times10$ cells (i.e., $n$ is fixed at 10), and each line segment within a cell is divided into $k$ sub-segments, where $k\in\{10,20,30,40,50\}$.
LB-A* is invoked to solve the formulated LBP and finds a trajectory $p$, which is visualized as the green lines in Fig.~\ref{lbmp:fig:converge} (a) and (c).
Let $C_{green},C_{blue}$ denote the cost of the green and the blue trajectories within a cell.
In Fig.~\ref{lbmp:fig:converge} (a), the cost of the edge between vertices $v_1$ and $v_2$ is $C_{green} = \frac{\sqrt{2}(k-1)}{10k}$, while the portion of $p^*$ contained in this cell has cost $C_{blue} = \frac{\sqrt{2}}{10}$, and $\frac{C_{green}}{C_{blue}}=\frac{k-1}{k}$.
Similarly, in Fig.~\ref{lbmp:fig:converge} (c), this ratio can also be calculated as $\frac{C_{green}}{C_{blue}}=\frac{k-1}{k}$ (note that $cost(v_4,v_5)=0$).
The same type of analysis can be applied for each cell along $p^*$, and we have $\frac{cost(p)}{C^*}=\frac{k-1}{k}$.
Thus in this experiment, by increasing $k$, $cost(p)$ converges to $C^*$.
As shown in Fig.~\ref{lbmp:fig:converge} (d), the numerical results output by LB-A* align with the above discussion.

\subsection{Experiment 2: One Dynamic Obstacle}
In the presence of dynamic obstacles, in general, $C^*$ is hard to obtain and the baseline approach is able to compute a lower bound as the dynamic obstacles are ignored.
To obtain feasible solutions (whose costs are upper bounds of $C^*$), we implement two algorithms, SIPP~\cite{phillips2011sipp} and RRT~\cite{lavalle2006planning}.
For SIPP, the workspace is discretized as a $40\times40$ eight-connected grid.
For RRT, we implement the basic version without any improving technique.

We begin with an instance (Fig.~\ref{lbmp:fig:oneObst} (b)) where a circular obstacle of radius $0.25$ moves from the center of the workspace to the left, while the robot moves from the middle point on the left border of the workspace to the middle point on the right border.
In this instance, the shortest path in the visibility graph constructed is simply a straight line connecting $s,d$, whose length is the Euclidean distance between $s,d$.
The lower bound of $C^*$ computed using the baseline approach is $1/0.03=33.3$.

\begin{figure}[tb]
	\centering
	\includegraphics[width=\linewidth]{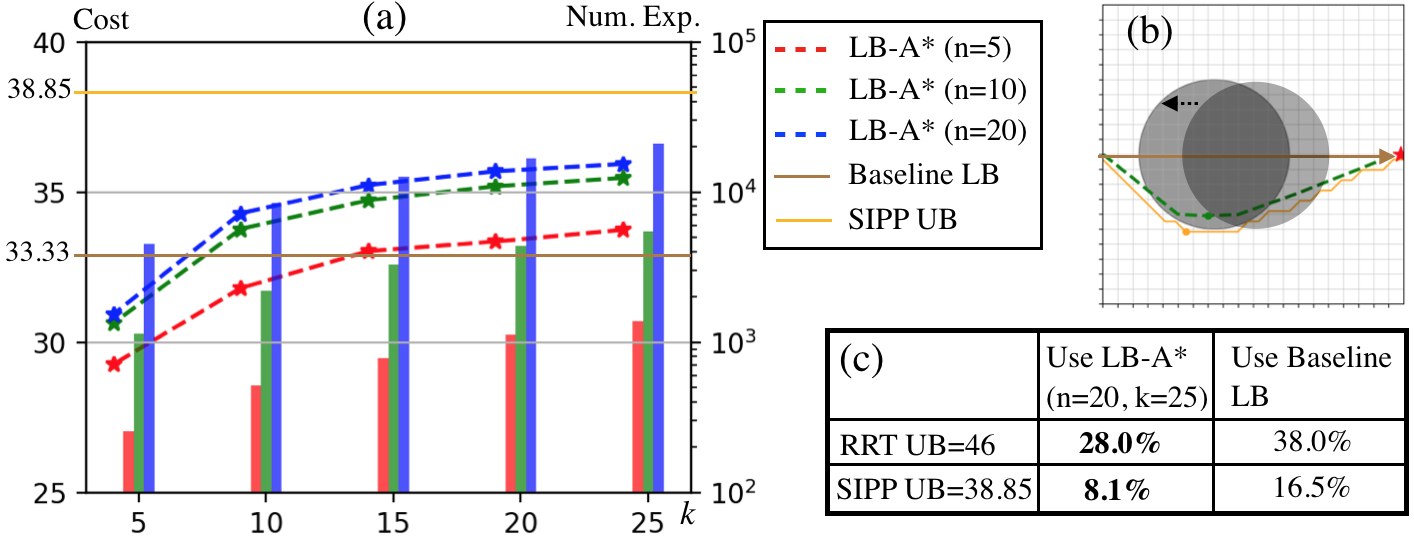}
	\caption{Experiment 2.
	Fig. (b) illustrates the test instance.
	The curves in Fig. (a) show (i) the lower bounds of $C^*$ computed by LB-A* with different $k,n$ parameters, (ii) the lower bound of $C^*$ computed by the baseline approach, (iii) the upper bound of $C^*$ computed by SIPP.
	The bar plot in Fig. (a) shows the number of states expanded by LB-A* during the search with different $k,n$ parameters.
	Fig. (c) compares the optimality bound estimated for the feasible solutions computed by RRT and SIPP. This estimate is computed by using either our approach or the baseline, and our approach can provide up to twice tighter optimality bound estimate (the less the better) than the baseline.}
	\label{lbmp:fig:oneObst}
\end{figure}

Fig.~\ref{lbmp:fig:oneObst} (a) shows the lower bounds computed by LB-A* with varying $n,k$.
When $n,k$ increase, the lower bound becomes larger (i.e., tighter).
Additionally, the lower bounds computed by LB-A* outperform the baseline as $n,k$ increase.
Given a feasible solution with cost $C$ ($C\geq C^*$) and a lower bound $C'$ ($C' \leq C^*$), an estimate of the optimality bound of $C$ can be computed as $\frac{C-C'}{C'}$, which justifies how far $C$ is away from $C^*$ in the worst case. (Note that $\frac{C-C^*}{C^*} \leq \frac{C-C'}{C'}$ since $C' \leq C^*$.)
Fig.~\ref{lbmp:fig:oneObst} (c) shows the estimate provided by using the lower bounds computed by LB-A* for both the feasible solutions computed by SIPP and RRT.
For example, the solution $p$ computed by SIPP has a cost of $38.85$, while the lower bound computed by LB-A* is $35.93$, and $p$ is guaranteed to be less than $\frac{38.85-35.93}{35.93}=8.1\%$ away from $C^*$.
As shown in Fig.~\ref{lbmp:fig:oneObst} (c), in comparison with the baseline, our approach can improve this optimality bound estimate from $16.5\%$ to $8.1\%$ (the less the better).

\subsection{Experiment 3: Dynamic and Static Obstacles}

We then consider an instance as shown in Fig.~\ref{lbmp:fig:tenObst} (b).
There are two bar-like static obstacles (the blue rectangles with negligible width), and there are 10 circular obstacles of radius $0.15$ moving from the center of the workspace along some random trajectories.
In this instance, the solution obtained by the baseline is visualized as the brown dashed lines in Fig.~\ref{lbmp:fig:tenObst} (b), whose cost is 37.16.
The lower bound computed by LB-A* outperforms the baseline when $n$ and $k$ increase, as shown in Fig.~\ref{lbmp:fig:tenObst} (a).
The optimality bound estimated for both the RRT and SIPP solutions using our approach is obviously better than using the baseline (Fig.~\ref{lbmp:fig:tenObst} (c)).

\begin{figure}[t]
	\centering
	\includegraphics[width=\linewidth]{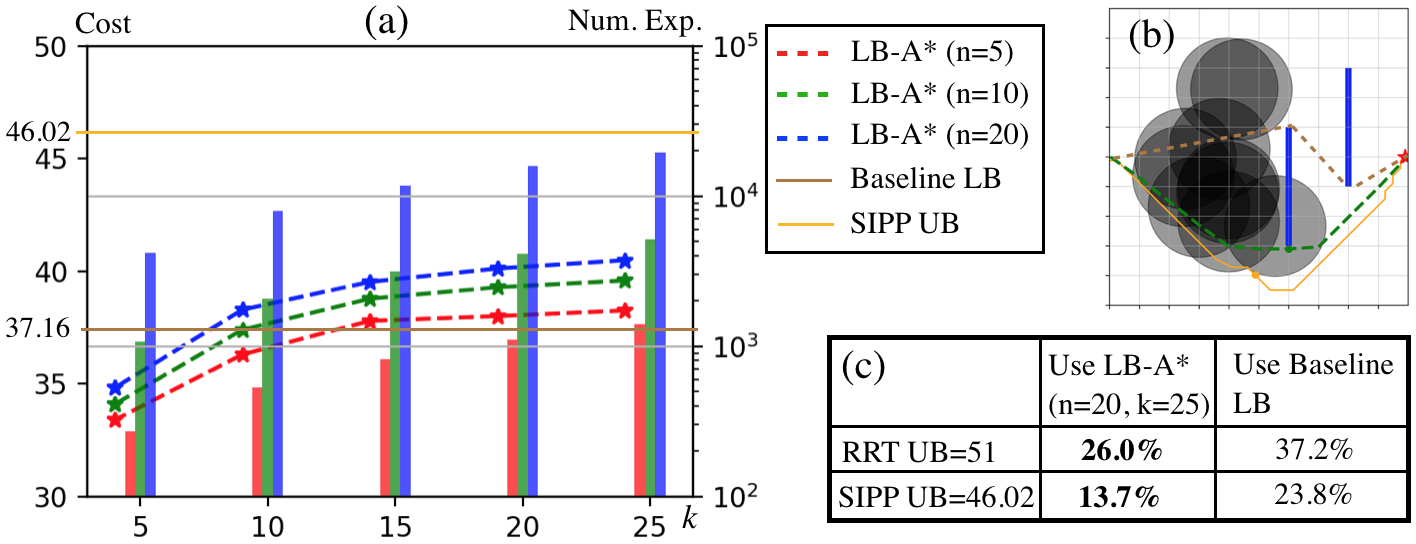}
	\caption{Experiment 3. See the caption of Fig.~\ref{lbmp:fig:oneObst} for details. The optimality bound estimated by using our approach is obviously tighter than using the baseline approach.}
	\label{lbmp:fig:tenObst}
\end{figure}

\subsection{With and Without the Expansion Constraint}
{\blue
So far, the reported results of LB-A* are obtained by enforcing the expansion constraint (Sec.~\ref{lbmp:sec:discuss:exp_cstr}) during the search.
Without the expansion constraint, the lower bounds obtained by LB-A* are slightly looser. For example, let $n=20,k=25$, for Experiment 2, the lower bound obtained without (and with) the expansion constraint is 35.82 (and 35.93 respectively).
It remains an open and challenging question about how to further improve the tightness of the lower bound and we will discuss potential future work in Sec.~\ref{lbmp:sec:conclude}.
}

\subsection{Computational Burden}
In Fig.~\ref{lbmp:fig:oneObst} (a) and Fig.~\ref{lbmp:fig:tenObst} (a), the number of expansions required by LB-A* during the search is shown as bar plots with respect to the right-side vertical axes of the plots.
With larger $n$ and $k$, the generated graph $G$ has more vertices and edges as explained in Sec.~\ref{lbmp:sec:analysis:compute}, which burdens the LB-A* search.
There is a trade-off between the search efficiency and the tightness of the lower bound computed.

{\blue
Finally, if the problem instance is simple (e.g. has few dynamic obstacles, similar to Experiment-1), the baseline approach can probably provide a tight lower bound with little computational effort, and our approach is not advantageous when considering the required computational effort.
}

\section{Conclusion}\label{lbmp:sec:conclude}

This work considers a Motion Planning Problem with Dynamic Obstacles (MPDO), and aims at computing tight lower bounds of the true optimum $C^*$.
To this end, a framework is developed, which consists of two parts: the first part is a bi-level discretization approach to formulate a lower bounding problem (LBP) corresponding to MPDO such that the solution cost to the LBP is guaranteed to be a lower bound of $C^*$; the second part of the framework is a graph search algorithm LB-A* that can solve the formulated LBP to obtain lower bounds.
We analyze and numerically evaluate the framework.
In our experiments, the lower bounds computed by our approach is tighter than using a baseline method.
Consequently, a tighter optimality bound estimate for the feasible solutions computed by SIPP and RRT can be obtained.

\vspace{2mm}
{\blue
\noindent\textbf{Future work} can follow many research directions.
First, one can improve the proposed LBP formulation in this work to reduce the computational complexity and memory usage, or improve the tightness of the lower bounds.
For example, the current relaxation of the obstacle avoidance constraint (i.e., Relaxation-1,2) can be potentially improved when additional knowledge about the obstacle trajectories (e.g. a polynomial of bounded degree~\cite{halperin2017algorithmic}) is known, which can help tighten the computed lower bound or help with the theoretic analysis of the computational complexity.
One can also consider non-uniform discretization~\cite{gochev2011path,ren2017deformed} or sampling-based strategies to adaptively discretize the workspace based on the trajectory of the dynamic obstacles, which have the potential to reduce the computational burden.
Another direction is to develop new planners that can solve the formulated LBP more efficiently by designing informative heuristics, or improving the search algorithm itself (such as the expansion constraint in Sec.~\ref{lbmp:sec:discuss:exp_cstr}).
Finally, one can consider extending the framework to non-Euclidean or high-dimensional spaces.
}

\section*{Acknowledgments}
This material is based upon work supported by the National Science Foundation under Grant No. 2120219 and 2120529. Any opinions, findings, and conclusions or recommendations expressed in this material are those of the author(s) and do not necessarily reflect the views of the National Science Foundation.


\bibliographystyle{plain}
\bibliography{ref}

%
%
%
%
%

\end{document}